\documentclass{article}
\usepackage{amsmath}

     \PassOptionsToPackage{numbers, compress}{natbib}

 \usepackage[final,main]{neurips_2026}

\usepackage{adjustbox} 
\usepackage{array}     
\usepackage[utf8]{inputenc} 
\usepackage[T1]{fontenc}    
\usepackage{hyperref}       
\usepackage{etoc}
\usepackage{url}            
\usepackage{booktabs}       
\usepackage{amsfonts}       
\usepackage{nicefrac}       
\usepackage{microtype}      
\usepackage{xcolor}         
\usepackage{graphicx}
\usepackage{multirow}
\usepackage{wrapfig}

\usepackage{indentfirst}
\title{All Roads Lead to Rome: Flow-driven Multi-Anchor Exploration for Open-Environment Active 3D Mapping}
\usepackage[ruled,vlined,linesnumbered]{algorithm2e}
\usepackage{colortbl}
\usepackage[table]{xcolor}
\definecolor{customcolor}{HTML}{CCE8CF} 
\usepackage{bm}
\usepackage{amssymb}
\usepackage{subcaption}  
\usepackage{caption}
\usepackage{placeins}
\usepackage{algorithm}
\usepackage{algorithmic}



\hypersetup{
    colorlinks=true,
    linkcolor=blue,
    citecolor=blue,
    urlcolor=black
}

\author{
Yang Li\textsuperscript{1} \quad
Aming Wu\textsuperscript{2}\thanks{Corresponding authors: Aming Wu and Yahong Han.} \quad
Zihao Zhang\textsuperscript{1} \quad
Ziju Han\textsuperscript{3} \quad
Sijia Zhang\textsuperscript{1} \quad
Yahong Han\textsuperscript{1}\footnotemark[1]
\\[0.6em]
\textsuperscript{1}School of Artificial Intelligence, Tianjin University, China
\\
\textsuperscript{2}School of Computer Science and Information Engineering, Hefei University of Technology, China
\\
\textsuperscript{3}School of Cybersecurity, Tianjin University, China
\\[0.4em]
{\tt\small
\{liyang1389, zhangzihao2490, hanziju, 3024244296, yahong\}@tju.edu.cn
}
\\
{\tt\small amwu@hfut.edu.cn}
}

\usepackage{amsthm}
\usepackage{mdframed}
\usepackage{etoolbox}
\usepackage{float}
\newtheoremstyle{neuripsplain}%
  {6pt}   
  {6pt}   
  {\itshape} 
  {}      
  {\bfseries} 
  {.}     
  {0.5em} 
  {}      

\theoremstyle{neuripsplain}
\newtheorem{theorem}{Theorem}[section]
\newtheorem{proposition}[theorem]{Proposition}

\newmdenv[
  skipabove=3pt,
  skipbelow=7pt,
  linewidth=0.3pt,
  linecolor=black,
  backgroundcolor=white,
  roundcorner=0pt,
  innertopmargin=2pt,
  innerbottommargin=2pt,
  innerleftmargin=10pt,
  innerrightmargin=10pt,
  splittopskip=\topskip,
  splitbottomskip=6pt
]{neuripsbox}

\BeforeBeginEnvironment{theorem}{\begin{neuripsbox}}
\AfterEndEnvironment{theorem}{\end{neuripsbox}}

\BeforeBeginEnvironment{proposition}{\begin{neuripsbox}}
\AfterEndEnvironment{proposition}{\end{neuripsbox}}

\BeforeBeginEnvironment{lemma}{\begin{neuripsbox}}
\AfterEndEnvironment{lemma}{\end{neuripsbox}}

\BeforeBeginEnvironment{corollary}{\begin{neuripsbox}}
\AfterEndEnvironment{corollary}{\end{neuripsbox}}

\begin{document}

\maketitle

\begin{abstract}
To advance the development of embodied intelligence, Open-Environment Active 3D Mapping has attracted increasing attention, aiming to perform a long-horizon and shortest trajectory exploration for reconstructing unseen scenarios. Since only limited information about unseen environments is available, methods built on the closed-set assumption, i.e., assuming that the test environments are similar to those seen during training, cannot generalize satisfactorily. In existing active mapping methods, long-horizon exploration is often guided by predicting a coarse long-range goal and then converting it into an executable path. However, this stage is usually formulated as single-point prediction. Under partial observability, the same local observation may correspond to multiple plausible exploration directions, making such deterministic prediction prone to brittle decisions and degraded performance in unseen scenarios. Our experiments further verify that this is a key factor underlying their weak generalization.
To address this issue, we reformulate long-horizon target prediction as conditional multimodal anchor generation using Conditional Flow Matching.
Instead of predicting a single goal, our method learns a conditional distribution over coarse exploration anchors from the current mapping state. These anchors are first converted into executable candidate paths through obstacle-aware planning. We then apply exploration-mode clustering to compress geometrically similar trajectories and reduce candidate redundancy. Finally, a hierarchical selection module selects the most promising mode and reranks paths within it to produce the final executable trajectory. Experiments show that our method improves generalization and reconstruction efficiency in open environments.


\end{abstract}

\section{Introduction}
\noindent 
Active 3D Mapping is a core problem in robotics and digital-twin applications, where an embodied agent must actively move, acquire observations, and reconstruct the geometry of an unknown scene as efficiently as possible. Unlike passive reconstruction~\cite{traditional3D-1,traditional3D-2}, active mapping \cite{georgakis2022uncertainty,  yan2023active} requires not only accurate surface reconstruction but also effective exploration decisions under a limited sensing budget. Thus, the key challenge lies not only in \emph{how to reconstruct}, but also in \emph{how to plan}.
Early works~\cite{NBV,OccAnt,MACARONS} are typically built on the closed-set assumption, i.e., assuming that the test environments are similar to those seen during training. However, in open environments, only limited information about unseen scenes is available at test time, which makes such methods difficult to generalize satisfactorily.

Recently, open-environment benchmarks~\cite{NBP,gleam} have extended active 3D mapping to cross-scene and cross-difficulty settings, placing much stronger demands on generalization. 
In these settings, long-horizon exploration is often guided by first predicting a coarse long-range goal. 
However, existing methods on these benchmarks~\cite{NBP,gleam} still largely formulate this stage as single-point prediction. This formulation is fundamentally limited under partial observability, where the same local observation may admit multiple plausible long-horizon exploration directions.
As shown in Fig.~\ref{fig:1}, single-point regression collapses multiple plausible long-horizon futures into one dominant goal. When the unseen scene continuation changes at test time, it tends to commit to a familiar branch, leading to brittle decisions and weak generalization.

This observation motivates us to move beyond single-goal prediction and instead model long-horizon exploration with multiple plausible goal hypotheses.
Recent progress in generative modeling~\cite{CFM,diffusion} suggests that conditional generation is particularly effective for representing multimodal futures. 
In particular, Conditional Flow Matching (CFM)~\cite{CFM} learns a continuous transport from a simple prior to a conditional target distribution, making it both stable and efficient for multimodal modeling.
However, existing flow-based models in embodied decision-making tasks~\cite{flow1,flow2,flow3,goalflow} are primarily designed for high-dimensional trajectory or action generation, whereas active 3D mapping fundamentally requires modeling long-horizon exploration intent under partial observability rather than directly synthesizing full behaviors.
\begin{figure}[t]
\centerline{\includegraphics[width=\columnwidth]{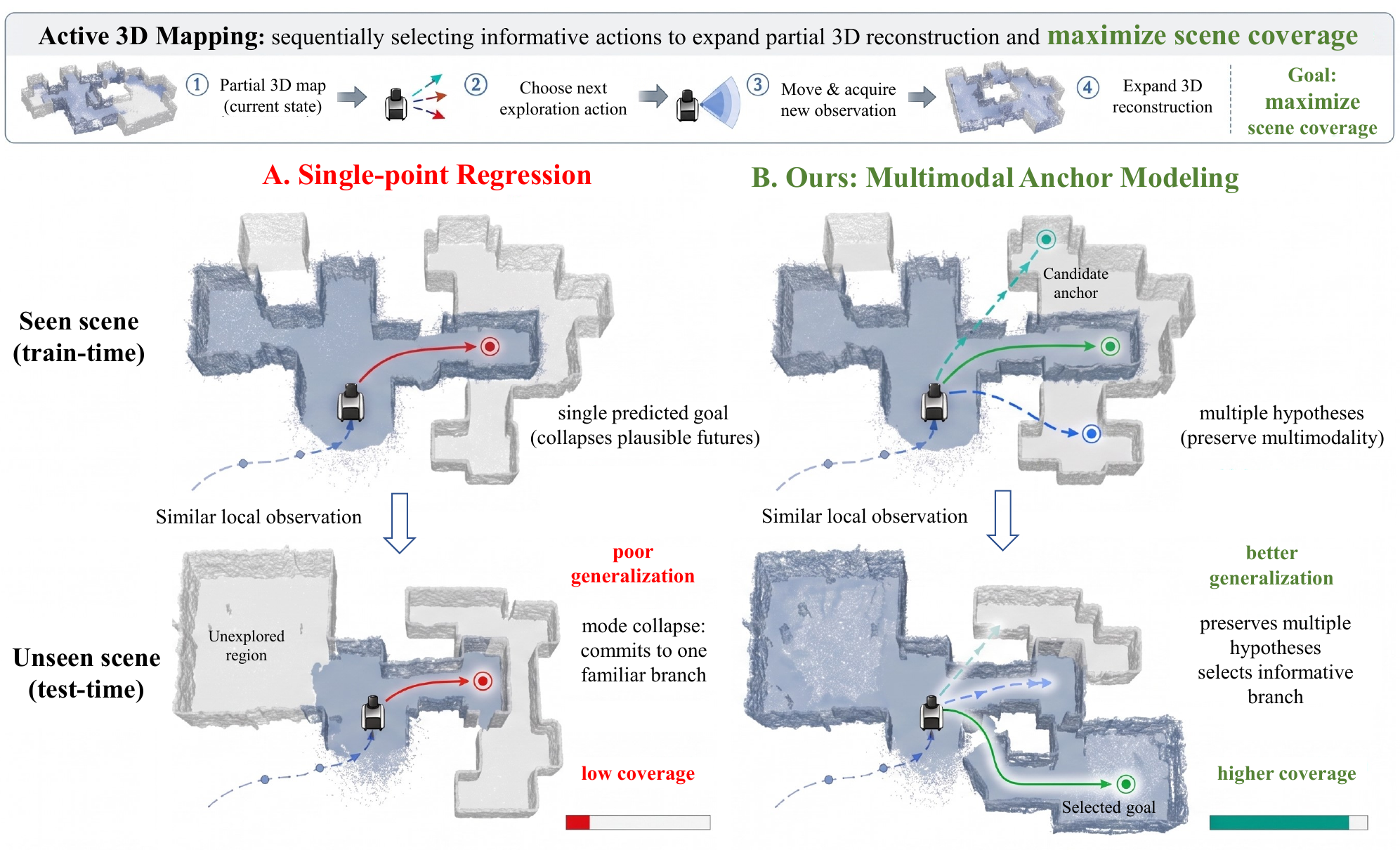}}
\vspace{-5pt}
\caption{\textbf{Motivation.} Under partial observability, a similar local observation may support multiple plausible long-horizon exploration directions. Single-goal prediction collapses these futures into one dominant branch and generalizes poorly when the unseen scene continuation changes. In contrast, multimodal anchor modeling preserves alternative hypotheses before downstream planning, leading to better generalization and higher coverage.}
\label{fig:1}
\end{figure}

To this end, we reformulate long-horizon target prediction in active 3D mapping as a conditional multimodal proposal generation problem, and adapt Conditional Flow Matching (CFM) to this setting in a task-specific manner so as to model multiple plausible long-horizon exploration hypotheses in a low-dimensional anchor space. The sampled anchors are then converted into executable candidate paths through obstacle-aware planning, compressed into exploration modes by anchor-guided clustering, and finally resolved by a hierarchical selection module to produce the final trajectory. Beyond the architectural design, we provide theoretical analysis showing why deterministic single-goal prediction is intrinsically mismatched to multimodal long-horizon exploration under partial observability, and why CFM offers a principled way to learn the resulting conditional anchor distribution. Extensive experiments, including within-dataset difficulty shift and cross-dataset transfer, demonstrate that the proposed framework achieves more robust generalization, higher reconstruction quality, and more efficient exploration in open environments.

Our contributions are threefold.
\textit{\textbf{First}}, we identify deterministic single-goal prediction as
a key limitation of open-environment active 3D mapping under partial
observability, and reformulate long-horizon exploration as conditional
multimodal proposal generation.
\textit{\textbf{Second}}, we propose a flow-driven multi-anchor exploration
framework that learns diverse long-horizon exploration anchors with
Conditional Flow Matching, followed by obstacle-aware planning,
exploration-mode clustering, and hierarchical path selection.
\textit{\textbf{Third}}, experiments on AiMDoom and Matterport3D under
within-dataset difficulty shift and cross-dataset transfer demonstrate
consistent improvements in generalization, reconstruction quality, and
exploration efficiency.

\section{Active 3D Mapping with Open-Environment Generalization}

Let $\mathcal S$ denote an unknown scene, and let $\mathcal X^{\mathrm{GT}}\subset\mathbb R^3$ denote its ground-truth surface. Active 3D mapping seeks to control an agent to acquire observations and reconstruct the scene as completely as possible under a finite budget $T$. At each time step $t$, the agent receives an RGB-D observation $I_t$ and predicts the next pose $c_t=(c_t^{\mathrm{pos}},c_t^{\mathrm{rot}})$ based on the history $\mathcal H_t=\{(I_0,c_0),\ldots,(I_t,c_t)\}$. This defines a partially observable sequential decision problem, where the high-level action satisfies $a_t\sim\pi(\cdot\mid\mathcal H_t)$.
In this work, $a_t$ is instantiated as a long-horizon exploration anchor, and executable candidate paths are recovered from sampled anchors using an explicit obstacle-aware planner. 

Beyond within-distribution mapping performance, we focus on generalization to unseen scenes in open environments. Let $p_{\mathrm{tr}}(\mathcal S)$ and $p_{\mathrm{te}}(\mathcal S)$ denote the training and test scene distributions, where $p_{\mathrm{tr}}(\mathcal S)\neq p_{\mathrm{te}}(\mathcal S)$ is allowed. Given the reconstructed point cloud $\mathcal P_t$, we measure mapping quality by the surface coverage $\operatorname{Cov}(\mathcal P_t,\mathcal X^{\mathrm{GT}})$, and use $\operatorname{AUC}(\pi_\theta)=\frac{1}{T}\sum_{t=1}^{T}\operatorname{Cov}(\mathcal P_t,\mathcal X^{\mathrm{GT}})$ to quantify intermediate mapping efficiency. The overall objective is
\begin{equation}
   \max_{\theta}\ \mathbb E_{\mathcal S\sim p_{\mathrm{te}}(\mathcal S)}
   \big[
   \lambda_1\operatorname{Cov}(\mathcal P_T,\mathcal X^{\mathrm{GT}})
   +\lambda_2\operatorname{AUC}(\pi_\theta)
   \big].
\end{equation}

\section{Method}
\subsection{Overview}
\label{sec:difficulty_invariant_context_encoding}

Active mapping in open environments exhibits two key challenges. First, the same local observation may admit multiple plausible long-horizon exploration directions under partial observability. Second, after geometric planning, these hypotheses often produce many similar candidates, which increases redundancy in downstream decision-making. To address this, we decompose long-horizon exploration into a hierarchical process.
We first represent the current mapping progress by a 2D projected state. Specifically, let $\mathcal P_t$ denote the accumulated reconstructed point cloud at time $t$. We partition $\mathcal P_t$ into $K$ horizontal slices and project them into $K$ local density maps, while the trajectory history is projected into a visitation map, yielding $E_t=\{I^{pc}_{t,1},\ldots,I^{pc}_{t,K},I_t^{\mathrm{hist}}\}$.
The shared context feature is then obtained by
$z_t = f_\phi(E_t)$,
where $f_\phi$ is a mapping-progress encoder. Based on $z_t$, we predict a local obstacle map $\hat{O}_t = g_\psi(z_t)$, which is used for subsequent feasible path planning.
Built on the context feature $z_t$, our method proceeds in four stages, as illustrated in Fig.~\ref{fig:overview}:
\begin{equation}
\label{eq:method_pipeline}
E_t \xrightarrow{\text{Encoder}} z_t \xrightarrow{\text{CFM}} \mathcal{A}_t \xrightarrow{\text{Planning+Clustering}} \{\mathcal{M}_t^{(m)}\}_{m=1}^{M_t} \xrightarrow{\text{Selection}} \tau_t^\star.
\end{equation}
Conditional Flow Matching (CFM) first generates a multimodal anchor set $\mathcal{A}_t$ for long-horizon exploration. The anchors are then converted into candidate paths, clustered into exploration modes $\{\mathcal{M}_t^{(m)}\}_{m=1}^{M_t}$, and hierarchically selected to obtain the final path $\tau_t^\star$.

\subsection{Conditional Flow Matching for Multimodal Anchors}
\label{sec:cfm_multimodal_anchors}

The purpose of this subsection is two-fold. First, we show why deterministic single-goal prediction is intrinsically mismatched to multimodal long-horizon exploration under partial observability. Second, based on this observation, we show why learning a conditional distribution over coarse anchors naturally leads to a CFM-based formulation in a low-dimensional decision space.

\textbf{Revisiting the learning target.}
In open-environment active 3D mapping, the same local observation may correspond to multiple plausible long-horizon exploration directions. Let $p(a\mid z_t)$ denote the conditional distribution of feasible long-horizon exploration anchors under the current context $z_t$. Suppose one instead learns a deterministic predictor $\hat a(z_t)$ under the standard squared loss:
\begin{equation}
\label{eq:deterministic_anchor_risk}
\mathcal{R}(\hat a;z_t)
=
\mathbb{E}_{a\sim p(a\mid z_t)}
\left[
\|\hat a(z_t)-a\|_2^2
\right].
\end{equation}

\begin{proposition}[Deterministic single-goal collapse]
\label{prop:deterministic_anchor_collapse}
The Bayes-optimal deterministic predictor of Eq.~\ref{eq:deterministic_anchor_risk} is the conditional mean
\begin{equation}
\label{eq:deterministic_anchor_mean}
\hat a^\star(z_t)=\mathbb{E}[a\mid z_t].
\end{equation}
Moreover, when $p(a\mid z_t)$ is multimodal and its feasible support is nonconvex due to scene geometry, $\hat a^\star(z_t)$ may lie outside any high-density mode of $p(a\mid z_t)$ and may even become infeasible for downstream planning.
\end{proposition}

\begin{figure}[t]
\centerline{\includegraphics[width=0.9\columnwidth]{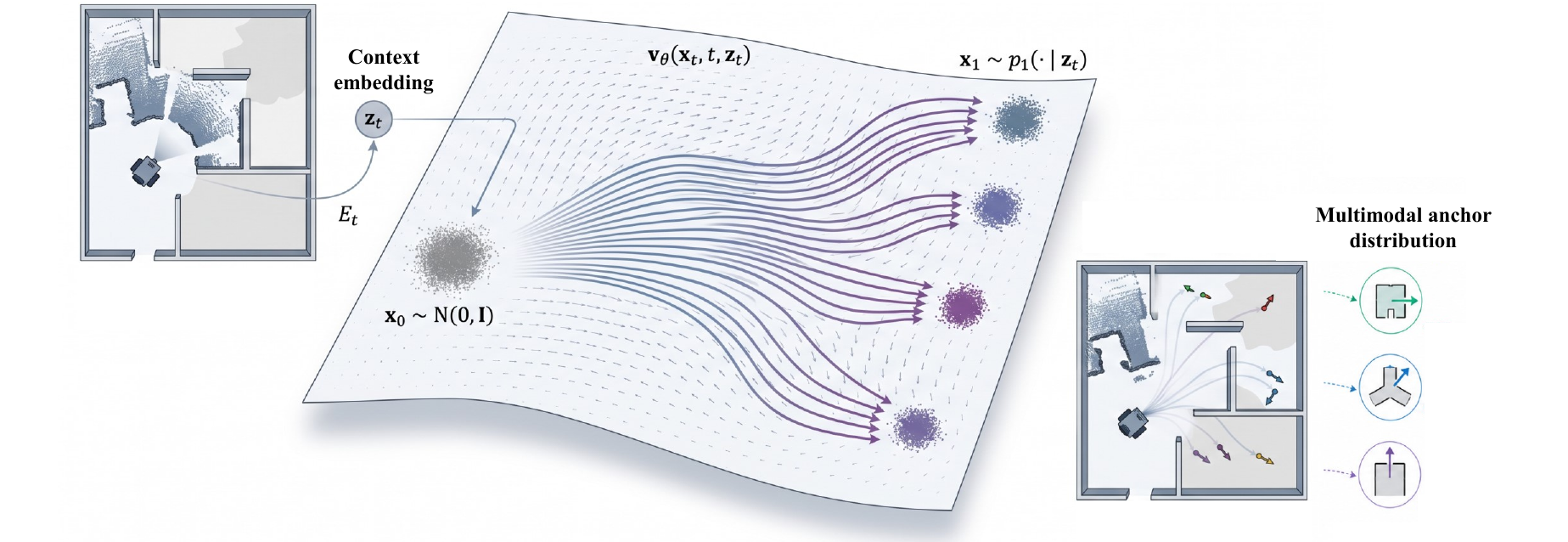}}
\vspace{-5pt}
\caption{\textbf{Conditional Flow Matching for multimodal long-horizon anchors.}
CFM transports samples from a simple prior $x_0\sim\mathcal N(0,I)$ toward the target anchor distribution $x_1\sim p_1(\cdot\mid z_t)$, yielding multiple plausible long-horizon exploration hypotheses under the same partial observation. Unlike directly generating full trajectories, this formulation models coarse exploration intent in a compact space and remains naturally compatible with downstream planning.}
\label{fig:cfm_anchor_transport}
\vskip -0.2in
\end{figure}

\textbf{\textit{Proof sketch.}}
Eq.~\ref{eq:deterministic_anchor_risk} is a standard conditional least-squares objective, whose minimizer is the conditional mean in Eq.~\ref{eq:deterministic_anchor_mean}. When the anchor distribution contains several separated feasible modes, however, this mean generally lies between them. Under obstacle-induced nonconvex geometry, such an averaged anchor can fall into a blocked or low-probability region, making it neither representative nor executable.

Proposition~\ref{prop:deterministic_anchor_collapse} formalizes the mismatch of single-goal prediction: under multimodal long-horizon futures, deterministic regression is mean-seeking rather than mode-preserving. Therefore, the correct learning target is not a single long-horizon goal, but the full conditional anchor distribution $p(a\mid z_t)$.

\textbf{A low-dimensional conditional anchor space.}
A way to preserve multimodality would be to generate full trajectories. However, the role of the long-horizon proposal is to express \emph{where to explore next} at a coarse semantic level, while detailed geometric feasibility is naturally handled by the explicit planner using the obstacle map $\hat O_t$. This suggests separating high-level long-horizon intent from low-level path realization, and introducing a compact variable for decision making.

We therefore define the long-horizon exploration anchor as
\begin{equation}
\label{eq:anchor_parameterization}
a = (a_{xy}, a_\theta) \in \mathbb{R}^4, \qquad
a_{xy} = (x_a, y_a), \quad
a_\theta = (\sin\theta_a, \cos\theta_a),
\end{equation}
where $a_{xy}$ and $a_\theta$ represent the position and orientation of the long-horizon goal, respectively. This low-dimensional parameterization isolates long-horizon exploration intent from detailed path geometry, leaving fine-grained execution to the downstream planner.

\textbf{Conditional transport via flow matching.}
Having established that the desired learning target is the conditional anchor distribution $p(a\mid z_t)$, we next need a principled way to model it. Conditional Flow Matching (CFM) is particularly suitable here, because it directly learns a conditional transport from a simple prior to the target distribution in continuous space, as illustrated in Fig.~\ref{fig:cfm_anchor_transport}.

During training, we set the source distribution as a standard Gaussian and the target samples as oracle anchors:
\begin{equation}
\label{eq:source_target_distributions}
x_0 \sim p_0=\mathcal{N}(0,I), \qquad
x_1 \sim p_1(\cdot\mid z_t),
\end{equation}
where $x_1$ is extracted from anchors in a high-reward path bank. Using the linear interpolation path
\begin{equation}
\label{eq:linear_interpolation_path}
x_t=(1-t)x_0+t x_1, \qquad t\sim\mathcal{U}(0,1),
\end{equation}
the target velocity is
\begin{equation}
\label{eq:target_velocity}
u_t=x_1-x_0.
\end{equation}
We then train a conditional velocity field $v_\theta(x_t,t,z_t)$ by minimizing
\begin{equation}
\label{eq:flow_matching_loss}
\mathcal{L}_{\mathrm{FM}}
=
\mathbb{E}_{x_0,x_1,t}
\left[
\|v_\theta(x_t,t,z_t)-u_t\|_2^2
\right].
\end{equation}

The key point is that Eq.~\ref{eq:flow_matching_loss} is not merely a heuristic diversity objective; it is a conditional regression problem whose population optimum recovers the transport field of the anchor distribution.

\begin{proposition}[CFM learns the conditional anchor transport field]
\label{prop:cfm_anchor_transport}
Let $p_t(\cdot\mid z_t)$ be the path induced by Eq.~\ref{eq:linear_interpolation_path} between $p_0$ and $p_1(\cdot\mid z_t)$. Then the population minimizer of Eq.~\ref{eq:flow_matching_loss} is
\begin{equation}
\label{eq:cfm_population_minimizer}
v^\star(x_t,t,z_t)=\mathbb{E}[u_t\mid x_t,t,z_t].
\end{equation}
Thus, CFM learns the conditional transport field induced by the entire anchor distribution $p_1(\cdot\mid z_t)$, rather than regressing a single deterministic long-horizon target.
\end{proposition}

\textbf{\textit{Proof sketch.}}
Eq.~\ref{eq:flow_matching_loss} is a standard conditional least-squares objective. By the orthogonal decomposition of regression risk,
\[
\mathbb{E}\!\left[\|v-u_t\|_2^2\right]
=
\mathbb{E}\!\left[\|v-\mathbb{E}[u_t\mid x_t,t,z_t]\|_2^2\right]
+
\mathbb{E}\!\left[\|u_t-\mathbb{E}[u_t\mid x_t,t,z_t]\|_2^2\right].
\]
The second term is independent of $v$, so the minimum is achieved by Eq.~\ref{eq:cfm_population_minimizer}.

Proposition~\ref{prop:cfm_anchor_transport} shows why CFM is a principled choice for our setting: once the learning target is reformulated from a deterministic goal to a conditional anchor distribution, CFM directly learns its conditional transport in the low-dimensional anchor space. Compared with directly generating full trajectories, this formulation is easier to optimize, preserves multimodal long-horizon hypotheses, and remains naturally compatible with explicit planning.

\begin{figure}[t]
\centerline{\includegraphics[width=\columnwidth]{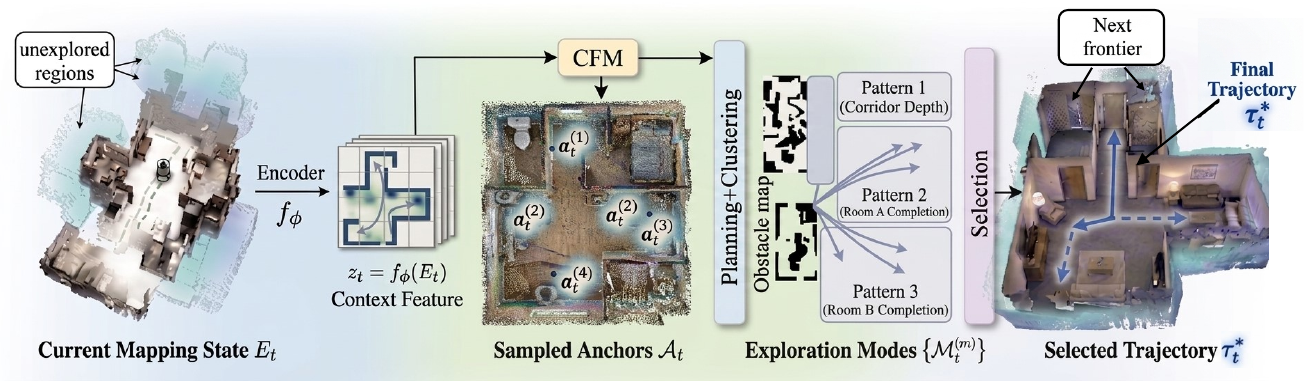}}
\vspace{-5pt}
\caption{\textbf{Overview of our method.} Given the current mapping state $E_t$, a mapping-progress encoder produces the context feature $z_t$. Conditional Flow Matching (CFM) then generates a multimodal set of long-horizon exploration anchors $\mathcal{A}_t$, which are converted into executable candidate paths and clustered into exploration modes $\{\mathcal{M}_t^{(m)}\}_{m=1}^{M_t}$. Finally, a hierarchical selection module first selects the most promising mode and then reranks paths within that mode to output the final trajectory $\tau_t^\star$.}
\label{fig:overview}
\vskip -0.2in
\end{figure}

During inference, we sample $K$ anchors from the learned conditional prior:
\begin{equation}
\label{eq:anchor_sampling}
\mathcal{A}_t=\{a_t^{(1)},\ldots,a_t^{(K)}\}, \qquad
a_t^{(k)}\sim p_\theta(a\mid z_t).
\end{equation}
These samples constitute the multimodal long-horizon exploration hypotheses at the current time step. To encourage coverage over distinct long-horizon exploration modes, we further introduce a diversity regularization term:
\begin{equation}
\label{eq:diversity_loss}
\mathcal{L}_{\mathrm{div}}
=
\frac{1}{K(K-1)}
\sum_{i\neq j}
\exp\left(
-\frac{\|a_t^{(i)}-a_t^{(j)}\|_2^2}{\sigma^2}
\right).
\end{equation}
Minimizing this term encourages different anchors to remain sufficiently dispersed, thereby covering diverse long-horizon exploration modes. In this way, the proposed CFM module provides a set of diverse, context-consistent exploration hypotheses for downstream planning and selection.

\subsection{Planning and Anchor-Guided Mode Construction.}
\label{sec:planning_mode_construction}
The previous subsection motivates modeling long-horizon exploration as a conditional anchor distribution rather than a single deterministic goal. However, anchors are only coarse intents and cannot be executed directly, and different anchors may still produce highly similar paths after planning. We therefore recover anchors into executable candidate paths via obstacle-aware planning, and further cluster them into representative exploration modes to reduce redundancy before hierarchical selection.

\textbf{From Anchors to Executable Paths.}
Given the current pose $c_t$, anchor $a_t^{(k)}$, and obstacle map $\hat{O}_t$, we obtain the candidate path via an explicit planner:

\begin{equation}
\label{eq:path_planning}
\tau_t^{(k)} = \textit{Plan}(c_t, a_t^{(k)}, \hat{O}_t),
\end{equation}
where $\textit{Plan}$ is implemented as an A*-based \cite{A*} shortest path planner. To characterize each candidate path for downstream reasoning, we further extract a compact path descriptor
\begin{equation}
\label{eq:path_descriptors}
\varphi(\tau_t^{(k)}) = [g_k,\ r_k,\ c_k,\ n_k,\ h_k,\ u_k],
\end{equation}
which respectively characterize the gain proxy, risk, cost, novelty, continuity deviation, and uncertainty statistics. They serve as structured intermediate features for the learned hierarchical selector in Sec.~\ref{sec:ood_hierarchical_selection}.

\textbf{From Candidate Paths to Exploration Modes.}
If selection is performed directly path-by-path on $\{\tau_t^{(k)}\}_{k=1}^{K}$, redundant candidates will interfere with decision-making. To circumvent this, we first discover more abstract exploration modes over the path set. Denoting the candidate path set as $\mathcal{T}_t = \{\tau_t^{(1)}, \ldots, \tau_t^{(K)}\}$, we define the distance between any two paths $\tau_i, \tau_j$ in the path space as:
\begin{equation}
\label{eq:path_distance}
d(\tau_i, \tau_j) = \lambda_1 d_{\mathrm{end}}(\tau_i, \tau_j) + \lambda_2 d_{\mathrm{shape}}(\tau_i, \tau_j) + \lambda_3 d_{\mathrm{prefix}}(\tau_i, \tau_j) + \lambda_4 d_{\mathrm{topo}}(\tau_i, \tau_j),
\end{equation}
where $d_{\mathrm{end}}$ measures the endpoint discrepancy, $d_{\mathrm{shape}}$ measures the geometric shape discrepancy of the paths after resampling, $d_{\mathrm{prefix}}$ measures the prefix trajectory discrepancy to characterize short-term execution behaviors, and $d_{\mathrm{topo}}$ measures the topological difference between the two paths in terms of obstacle circumvention.

Under the distance metric, we cluster the candidate paths into $M_t$ exploration modes:
\begin{equation}
\label{eq:mode_clustering}
\mathcal{T}_t \longrightarrow \{\mathcal{M}_t^{(1)}, \ldots, \mathcal{M}_t^{(M_t)}\}.
\end{equation}
Each mode $\mathcal{M}_t^{(m)}$ corresponds to a class of long-horizon exploration hypotheses, such as ``\textit{prioritize entering the left room}'', ``\textit{continue deep along the main corridor}'', or ``\textit{complete the current local region first before expanding farther}''. Furthermore, we define a representative path $\tilde{\tau}_t^{(m)}\in\mathcal{M}_t^{(m)}$ for each mode, which is implemented as the medoid within that mode.

Simultaneously, we define mode statistics:
\begin{equation}
\label{eq:mode_statistics}
\chi_t^{(m)} = \left[ |\mathcal{M}_t^{(m)}|, \ \bar{g}_m, \ \bar{r}_m, \ \operatorname{Var}(g)_m, \ \operatorname{Var}(r)_m, \ \rho_m \right],
\end{equation}
where $|\mathcal{M}_t^{(m)}|$ denotes the mode size, $\bar{g}_m$ and $\bar{r}_m$ denote the average gain and average risk within the mode, $\operatorname{Var}(g)_m$ and $\operatorname{Var}(r)_m$ measure the dispersion of gain and risk, and $\rho_m$ denotes the mode radius. Together with the representative path, these statistics provide a compact summary of the mode's quality and internal consistency. As with the path descriptors above, they are used as auxiliary structured inputs to the downstream learned scorer, rather than as standalone selection criteria.

Therefore, this module elevates multimodal anchor hypotheses into executable, comparable, and hierarchically structured exploration mode representations.

\subsection{Hierarchical Mode Selection}
\label{sec:ood_hierarchical_selection}
The previous subsection organizes multimodal anchors into a set of exploration modes, but mode construction alone does not resolve the final decision. In open-environment active 3D mapping, the agent's goal is to make robust long-horizon exploration decisions under partial observability and scene shift, rather than to score many redundant paths independently. If selection is still performed directly over all candidate paths, geometrically similar trajectories can interfere with scoring and bias the policy toward unstable local choices, which weakens generalization and reduces exploration efficiency. We therefore adopt a hierarchical selection strategy: first selecting a promising mode at the strategy level, and then reranking paths within that mode to obtain the final executable trajectory.

\textbf{Mode-Level Scoring.}
For each mode, we construct a mode representation:
\begin{equation}
\label{eq:mode_representation}
z_t^{(m)} = \Gamma(z_t, \tilde{\tau}_t^{(m)}, \chi_t^{(m)}),
\end{equation}
where $z_t$ is the shared context output by the context encoder, and $\Gamma(\cdot)$ represents the mode encoder used to fuse the shared context, representative path, and mode statistics.

In open environments, a mode worth executing should not only possess a high expected gain but also simultaneously satisfy the criteria of controllable risk, low uncertainty, and continuity with existing execution plans, while preserving a degree of exploration diversity when necessary. To this end, we predict multi-head scores via a mode scorer:
\begin{equation}
\label{eq:mode_scoring}
(G_m, R_m, U_m, C_m, D_m) = s_\eta(z_t^{(m)}),
\end{equation}
where $G_m$, $R_m$, $U_m$, $C_m$, and $D_m$ represent the expected gain, execution risk, uncertainty, mode switching or cost, and exploration diversity reward of the mode, respectively. Here, $s_\eta(\cdot)$ represents a learnable mode-level scoring network.
Accordingly, the overall mode-level score is defined as:
\begin{equation}
\label{eq:mode_overall_score}
S_m^{\mathrm{mode}} = w_G G_m - w_R R_m - w_U U_m - w_C C_m + w_D D_m,
\end{equation}
where $w_G, w_R, w_U, w_C, w_D \ge 0$ are learnable weights. Thus, the optimal mode selection at the current time step is:
\begin{equation}
\label{eq:mode_selection}
m_t^\star = \arg\max_m S_m^{\mathrm{mode}}.
\end{equation}

\textbf{Intra-Mode Local Reranking.}
Within the selected mode $\mathcal{M}_t^{(m_t^\star)}$, for any candidate path $\tau$, we construct a path-level representation:
\begin{equation}
\label{eq:path_representation}
\zeta_t(\tau) = \Lambda(z_t, \tau, \varphi(\tau), \chi_t^{(m_t^\star)}),
\end{equation}
where $\Lambda(\cdot)$ represents the feature fusion function of the local reranker. Subsequently, the local reranker outputs four path-level scoring heads $\hat{G}(\tau), \hat{R}(\tau), \hat{C}(\tau), \hat{H}(\tau)$, which respectively characterize path gain, path risk, path cost, and continuity penalty.
Consequently, the overall local path score is defined as:
\begin{equation}
\label{eq:path_overall_score}
S^{\mathrm{local}}(\tau) = \hat{G}(\tau) - \lambda_R \hat{R}(\tau) - \lambda_C \hat{C}(\tau) - \lambda_H \hat{H}(\tau).
\end{equation}

The final execution path is thus given by:
\begin{equation}
\label{eq:path_selection}
\tau_t^\star = \arg\max_{\tau\in\mathcal{M}_t^{(m_t^\star)}} S^{\mathrm{local}}(\tau).
\end{equation}
Thus, this section decouples long-range strategy reasoning from local path execution via hierarchical selection, improving stability and robustness in open-environment exploration.

\section{Experiments}

\subsection{Experimental Setup}
We evaluate our method on Matterport3D (MP3D)~\cite{mp3d} and AiMDoom~\cite{NBP}. For AiMDoom, we adopt a 70/30 train/test split within each difficulty level.
On AiMDoom, we report two standard active-mapping metrics: \emph{Final Coverage}, which measures the final scene coverage at the end of an episode, and \emph{AUC}, which evaluates reconstruction efficiency by the area under the coverage curve over time. Surface coverage is computed against the ground-truth mesh following~\cite{NBP}. We evaluate five trajectories per scene using identical random initial poses for different methods, and report the mean and standard deviation over all testing trajectories. For MP3D, we additionally report the completion metrics used in prior work, namely \emph{Comp. (\%)} and \emph{Comp. (cm)}.

We follow the same mapping-state construction and basic training protocol as~\cite{NBP} for fair comparison. Our model takes projected local map observations and trajectory history as input, and predicts an obstacle map for downstream proposal generation and planning.

\begin{table*}[b]
\centering
\footnotesize

\caption{\textbf{Evaluation results on AiMDoom Dataset.}
For each difficulty level, all baseline models, including ours, are trained from scratch on the corresponding training set to ensure a fair comparison.}
\label{tab:doom_main_experiments}
\tabcolsep=0.07cm
\begin{tabular*}{\textwidth}{@{\extracolsep{\fill}}l*{8}{c}}
\toprule
& \multicolumn{2}{c}{\textbf{Simple}} & \multicolumn{2}{c}{\textbf{Normal}} & \multicolumn{2}{c}{\textbf{Hard}} & \multicolumn{2}{c}{\textbf{Insane}} \\
\cmidrule(lr){2-3} \cmidrule(lr){4-5} \cmidrule(lr){6-7} \cmidrule(lr){8-9}
& Final Cov. & AUCs & Final Cov. & AUCs & Final Cov. & AUCs & Final Cov. & AUCs \\
\midrule
Random & 0.323\tiny{±0.156} & 0.270\tiny{±0.135} & 0.190\tiny{±0.124} & 0.152\tiny{±0.103} & 0.124\tiny{±0.082} & 0.088\tiny{±0.060} & 0.074\tiny{±0.048} & 0.050\tiny{±0.035} \\
FBE \cite{FBE} & 0.760\tiny{±0.174} & 0.605\tiny{±0.171} & 0.565\tiny{±0.139} & 0.415\tiny{±0.109} & 0.425\tiny{±0.114} & 0.311\tiny{±0.080} & 0.330\tiny{±0.097} & 0.239\tiny{±0.079} \\
SCONE \cite{scone} & 0.577\tiny{±0.173} & 0.483\tiny{±0.138} & 0.412\tiny{±0.114} & 0.313\tiny{±0.087} & 0.290\tiny{±0.093} & 0.210\tiny{±0.072} & 0.196\tiny{±0.079} & 0.140\tiny{±0.060} \\
MACARONS \cite{MACARONS} & 0.599\tiny{±0.200} & 0.479\tiny{±0.172} & 0.418\tiny{±0.120} & 0.314\tiny{±0.088} & 0.302\tiny{±0.097} & 0.218\tiny{±0.070} & 0.192\tiny{±0.078} & 0.139\tiny{±0.058} \\
NBP \cite{NBP} & 0.879\tiny{±0.142} & 0.692\tiny{±0.156} & 0.734\tiny{±0.142} & 0.526\tiny{±0.112} & 0.618\tiny{±0.153} & 0.432\tiny{±0.115} & 0.472\tiny{±0.095} & 0.312\tiny{±0.073} \\
\midrule
\textbf{Ours} & \textbf{0.913}\tiny{±0.138} & \textbf{0.707}\tiny{±0.141} & \textbf{0.781}\tiny{±0.125} & \textbf{0.554}\tiny{±0.096} & \textbf{0.662}\tiny{±0.139} & \textbf{0.446}\tiny{±0.112} & \textbf{0.534}\tiny{±0.073} & \textbf{0.325}\tiny{±0.089} \\
\bottomrule
\end{tabular*}

\vspace{0.8em}

\begin{minipage}[t]{0.48\textwidth}
\centering
\captionof{table}{\textbf{Comparison on the MP3D dataset.}}
\label{tab:comparison3}
\setlength{\tabcolsep}{2pt}
\begin{tabular*}{\linewidth}{@{\extracolsep{\fill}}lcc}
\toprule
\textbf{Method} & \textbf{Comp. (\%) $\uparrow$} & \textbf{Comp. (cm) $\downarrow$} \\
\midrule
Random & 45.67 & 26.53 \\
FBE \cite{FBE} & 71.18 & 9.78 \\
UPEN \cite{georgakis2022uncertainty} & 69.06 & 10.60 \\
OccAnt \cite{OccAnt} & 71.72 & 9.40 \\
ANM \cite{yan2023active} & 73.15 & 9.11 \\
NBP \cite{NBP} & 79.38 & 6.78 \\
\midrule
\textbf{Ours} & \textbf{82.67} & \textbf{5.67} \\
\bottomrule
\end{tabular*}
\end{minipage}
\hfill
\begin{minipage}[t]{0.48\textwidth}
\centering
\captionof{table}{\textbf{Train on AiMDoom test on MP3D.}}
\label{tab:cross_dataset_mp3d}
\setlength{\tabcolsep}{2pt}
\begin{tabular*}{\linewidth}{@{\extracolsep{\fill}}lcc}
\toprule
\textbf{Method} & \textbf{Comp. (\%) $\uparrow$} & \textbf{Comp. (cm) $\downarrow$} \\
\midrule
Random & 32.12 & 29.36 \\
FBE \cite{FBE} & 56.73 & 14.45 \\
UPEN \cite{georgakis2022uncertainty} & 54.70 & 13.09 \\
OccAnt \cite{OccAnt} & 58.49 & 14.18 \\
ANM \cite{yan2023active} & 60.22 & 11.78 \\
NBP \cite{NBP} & 61.81 & 10.22 \\
\midrule
\textbf{Ours} & \textbf{68.45} & \textbf{9.11} \\
\bottomrule
\end{tabular*}
\end{minipage}

\vspace{-0.8em}
\end{table*}

\subsection{Comparison with the State of the Art}
\textbf{AiMDoom dataset.}
Table~\ref{tab:doom_main_experiments} shows that our method achieves the best performance on all four AiMDoom difficulty levels. Compared with NBP~\cite{NBP}, it consistently improves both Final Coverage and AUC, with larger gains on the harder splits. In particular, Final Coverage improves from 0.734 to 0.781 on Normal, from 0.618 to 0.662 on Hard, and from 0.472 to 0.534 on Insane, while AUC also improves consistently. This trend is particularly meaningful because AiMDoom is designed to increase scene ambiguity and long-horizon planning difficulty as the level becomes harder. The results therefore suggest that preserving multiple plausible long-horizon hypotheses makes exploration more robust under increasing uncertainty, rather than committing prematurely to a single predicted branch.

\begin{figure}
\centerline{\includegraphics[width=\columnwidth]{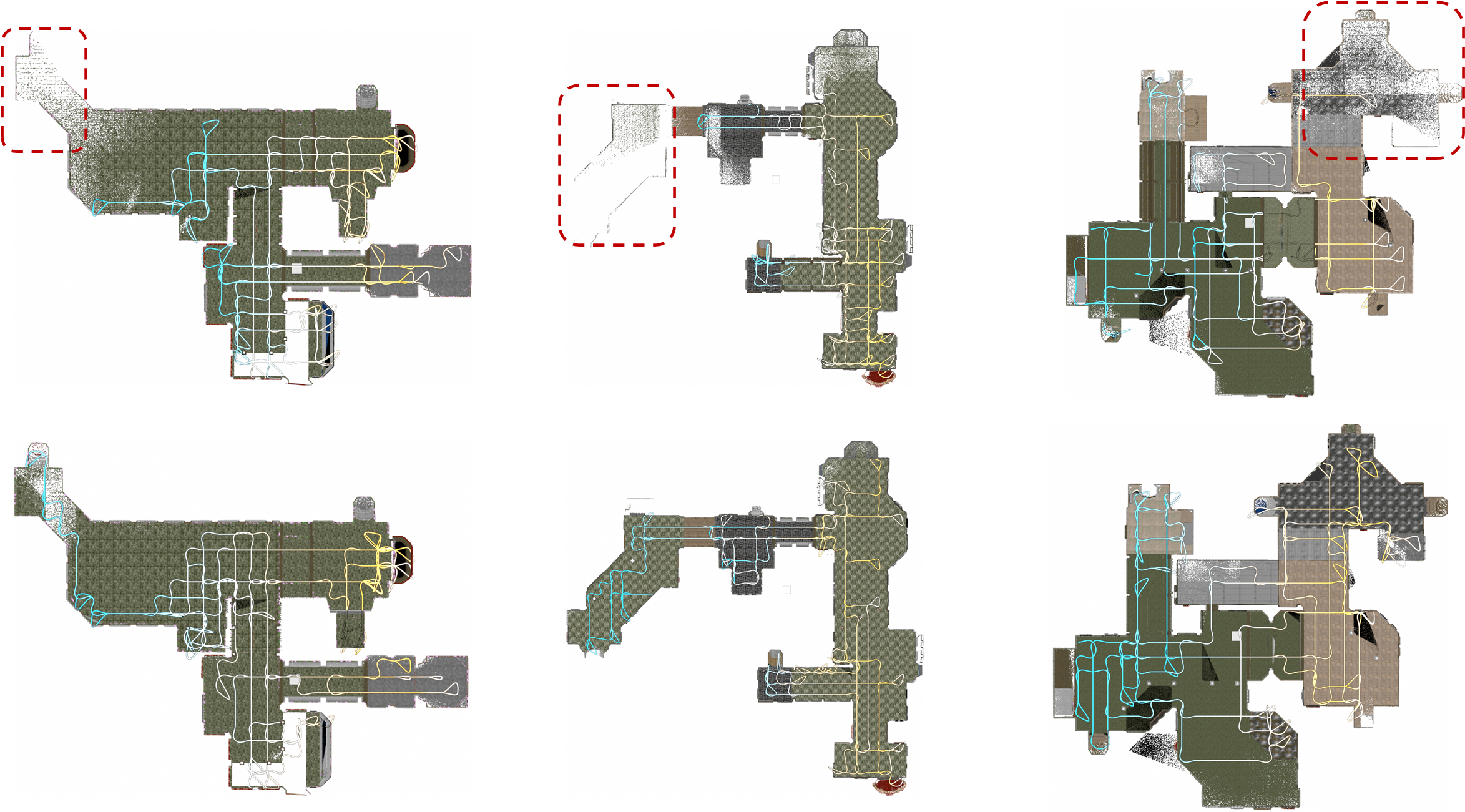}}
\vspace{-5pt}
\caption{\textbf{Qualitative comparison between the SOTA method NBP (Top) and Ours (Bottom).}
For each scene, both methods start from the same initial pose. The dashed boxes highlight representative regions that remain unexplored by NBP but are successfully recovered by our method. This difference is not merely local: NBP tends to commit early to a single dominant long-horizon branch, which induces a narrower global coverage pattern and leaves weakly indicated but still informative continuations unexplored. In contrast, our method preserves multiple plausible long-horizon exploration hypotheses before downstream planning and selection, enabling more balanced expansion across alternative scene continuations and ultimately yielding more complete scene coverage.}
\label{fig:sota-vis}
\vskip -0.2in
\end{figure}

\textbf{MP3D dataset.}
As shown in Table~\ref{tab:comparison3}, our method achieves the best results on MP3D, reaching 82.67 in Comp. (\%) and 5.67 in Comp. (cm). Compared with NBP~\cite{NBP}, this corresponds to a gain of 3.29 points in completion ratio and a reduction of 1.11 cm in completion error. The improvement in both metrics is important: the higher completion ratio indicates better scene coverage, while the lower completion error shows that the reconstructed geometry is also more accurate. These results suggest that the multimodal long-horizon modeling is not only effective on synthetic open-environment benchmarks, but also transfers well to realistic indoor scenes with more complex layouts.

\subsection{Generalization under Out-of-Distribution Settings}

We evaluate OOD generalization from two perspectives: \emph{within-dataset difficulty shift} on AiMDoom and \emph{cross-dataset transfer} from AiMDoom to MP3D. The results consistently support our central claim: preserving multiple long-horizon exploration hypotheses improves generalization, whereas single-goal methods generalize poorly under scene shift.

\begin{table}
\centering
\captionof{table}{\textbf{Generalization under difficulty shift on AiMDoom.}
We report two complementary protocols. 
Protocol A evaluates cross-difficulty transfer by training on the \textit{Simple} split only and directly testing on all difficulty levels . 
Protocol B evaluates leave-one-difficulty-out generalization by training on \textit{Simple} + \textit{Normal} + \textit{Hard} and directly testing on the held-out \textit{Insane} split. 
}
\label{tab:doom_generalization_all}
\vspace{2pt}
\scriptsize
\tabcolsep=0.06cm
\begin{tabular*}{\textwidth}{@{\extracolsep{\fill}}ll*{8}{c}}
\toprule
& & \multicolumn{2}{c}{\textbf{Simple}} & \multicolumn{2}{c}{\textbf{Normal}} & \multicolumn{2}{c}{\textbf{Hard}} & \multicolumn{2}{c}{\textbf{Insane}} \\
\cmidrule(lr){3-4} \cmidrule(lr){5-6} \cmidrule(lr){7-8} \cmidrule(lr){9-10}
\textbf{Protocol} & \textbf{Method} & Final Cov. & AUC & Final Cov. & AUC & Final Cov. & AUC & Final Cov. & AUC \\
\midrule

\multicolumn{10}{l}{\textbf{Protocol A: Train on Simple, test on all difficulty levels}} \\
\midrule
& Random & 0.323\tiny{$\pm$0.156} & 0.270\tiny{$\pm$0.135} & 0.174\tiny{$\pm$0.112} & 0.132\tiny{$\pm$0.091} & 0.101\tiny{$\pm$0.075} & 0.072\tiny{$\pm$0.054} & 0.062\tiny{$\pm$0.041} & 0.041\tiny{$\pm$0.030} \\
& FBE \cite{FBE} & 0.760\tiny{$\pm$0.174} & 0.605\tiny{$\pm$0.171} & 0.512\tiny{$\pm$0.155} & 0.364\tiny{$\pm$0.124} & 0.368\tiny{$\pm$0.132} & 0.255\tiny{$\pm$0.098} & 0.284\tiny{$\pm$0.110} & 0.188\tiny{$\pm$0.088} \\
& SCONE \cite{scone} & 0.577\tiny{$\pm$0.173} & 0.483\tiny{$\pm$0.138} & 0.354\tiny{$\pm$0.128} & 0.245\tiny{$\pm$0.104} & 0.213\tiny{$\pm$0.111} & 0.142\tiny{$\pm$0.082} & 0.145\tiny{$\pm$0.092} & 0.094\tiny{$\pm$0.071} \\
& MACARONS \cite{MACARONS} & 0.599\tiny{$\pm$0.200} & 0.479\tiny{$\pm$0.172} & 0.368\tiny{$\pm$0.141} & 0.252\tiny{$\pm$0.115} & 0.225\tiny{$\pm$0.124} & 0.148\tiny{$\pm$0.087} & 0.151\tiny{$\pm$0.096} & 0.098\tiny{$\pm$0.075} \\
& NBP \cite{NBP} & 0.879\tiny{$\pm$0.142} & 0.692\tiny{$\pm$0.156} & 0.695\tiny{$\pm$0.158} & 0.502\tiny{$\pm$0.128} & 0.562\tiny{$\pm$0.161} & 0.401\tiny{$\pm$0.122} & 0.434\tiny{$\pm$0.114} & 0.275\tiny{$\pm$0.098} \\
& \textbf{Ours} & \textbf{0.913}\tiny{$\pm$0.138} & \textbf{0.707}\tiny{$\pm$0.141} & \textbf{0.727}\tiny{$\pm$0.152} & \textbf{0.526}\tiny{$\pm$0.133} & \textbf{0.597}\tiny{$\pm$0.174} & \textbf{0.403}\tiny{$\pm$0.121} & \textbf{0.467}\tiny{$\pm$0.106} & \textbf{0.313}\tiny{$\pm$0.101} \\
\midrule

\multicolumn{10}{l}{\textbf{Protocol B: Train on Simple + Normal + Hard, test on held-out Insane}} \\
\midrule
& Random & -- & -- & -- & -- & -- & -- & 0.068\tiny{$\pm$0.045} & 0.048\tiny{$\pm$0.038} \\
& FBE \cite{FBE} & -- & -- & -- & -- & -- & -- & 0.312\tiny{$\pm$0.124} & 0.201\tiny{$\pm$0.105} \\
& SCONE \cite{scone} & -- & -- & -- & -- & -- & -- & 0.168\tiny{$\pm$0.101} & 0.112\tiny{$\pm$0.082} \\
& MACARONS \cite{MACARONS} & -- & -- & -- & -- & -- & -- & 0.182\tiny{$\pm$0.115} & 0.115\tiny{$\pm$0.088} \\
& NBP \cite{NBP} & -- & -- & -- & -- & -- & -- & 0.421\tiny{$\pm$0.128} & 0.302\tiny{$\pm$0.112} \\
& \textbf{Ours} & -- & -- & -- & -- & -- & -- & \textbf{0.438}\tiny{$\pm$0.119} & \textbf{0.317}\tiny{$\pm$0.089} \\
\bottomrule
\end{tabular*}
\vspace{-0.8em}
\end{table}

\textbf{Cross-dataset transfer from AiMDoom to MP3D.}
Table~\ref{tab:cross_dataset_mp3d} reports zero-shot transfer from AiMDoom to MP3D. Despite the large domain gap, our method achieves the best performance, reaching 68.45 in Comp. (\%) and 9.11 in Comp. (cm), improving over NBP by 6.64 points in completion ratio and 1.11 cm in completion error. This result suggests that multimodal long-horizon modeling transfers better across unseen scene distributions than single-goal prediction.

\textbf{Within-dataset difficulty shift on AiMDoom.}
Table~\ref{tab:doom_generalization_all} reports two OOD protocols on AiMDoom. Under both settings, our method consistently outperforms NBP, with especially clear gains on the harder splits. This suggests that preserving multiple plausible long-horizon exploration paths is more robust to increasing scene ambiguity than single-goal prediction, which tends to commit early to one branch and thus generalizes poorly under difficulty shift.

\subsection{Ablation Study}
To verify the generalization of our method to unseen open-world scenes, all experiments in this subsection are conducted under a cross-dataset setting: \textit{trained on AiMDoom and tested on MP3D}.

\paragraph{Component Analysis.}
Table~\ref{tab:nbp_progressive_ablation} reports a progressive validation of the proposed pipeline, where MAG, EMC, and HMS denote multimodal anchor generation, exploration-mode clustering, and hierarchical mode/path selection, respectively. As modules are introduced stage by stage, the performance improves consistently under the cross-dataset OOD setting. The largest gain comes from adding MAG, which improves Comp. (\%) from 61.81 to 66.21 and reduces Comp. (cm) from 10.22 to 9.64. This is consistent with our central claim that the main limitation of existing methods lies in deterministic single-goal prediction: replacing it with multimodal anchor generation already yields a much stronger long-horizon decision mechanism. Adding EMC further improves the results to 67.60 and 9.36, showing that organizing raw candidate paths into exploration modes is beneficial for reducing redundancy among geometrically similar proposals. Finally, adding HMS yields the best performance of 68.45 and 9.11, suggesting that hierarchical mode/path selection further stabilizes downstream decision making by decoupling high-level strategy choice from local path reranking.

\begin{wraptable}{r}{0.50\textwidth}
\vspace{-8pt}
\centering
\scriptsize
\setlength{\tabcolsep}{3pt}
\renewcommand{\arraystretch}{0.95}
\caption{\textbf{Progressive validation of the proposed pipeline under the cross-dataset OOD setting.}}
\label{tab:nbp_progressive_ablation}
\begin{tabular}{cccc|cc}
\toprule
\multicolumn{4}{c|}{Method} & \multicolumn{2}{c}{OOD Setting} \\
\midrule
\multicolumn{1}{c}{Baseline} &
\multicolumn{1}{c}{MAG} &
\multicolumn{1}{c}{EMC} &
\multicolumn{1}{c|}{HMS} &
\multicolumn{1}{c}{Comp. (\%) $\uparrow$} &
\multicolumn{1}{c}{Comp. (cm) $\downarrow$} \\
\midrule
$\checkmark$ &  &  &  & 61.81 & 10.22 \\
$\checkmark$ & $\checkmark$ &  &  & 66.21 & 9.64 \\
$\checkmark$ & $\checkmark$ & $\checkmark$ &  & 67.60 & 9.36 \\
$\checkmark$ & $\checkmark$ & $\checkmark$ & $\checkmark$ & \textbf{68.45} & \textbf{9.11} \\
\bottomrule
\end{tabular}
\vspace{-10pt}
\end{wraptable}

\section{Conclusion}

In this paper, we study active 3D mapping in open environments, where partial observability and scene shift make deterministic single-goal prediction difficult to generalize. To address this issue, we reformulate long-horizon decision making as a conditional multimodal proposal generation problem and propose a flow-matching-based framework for multimodal anchor modeling. The generated anchors are further converted into executable candidate paths, compressed into exploration modes, and selected through a hierarchical decision process. Extensive experiments show that our method consistently improves reconstruction quality, exploration efficiency, and robustness under both difficulty shift and cross-dataset transfer. These results suggest that preserving multimodal long-horizon exploration hypotheses is a promising direction for open-environment active mapping.

\paragraph{Limitations.} The current framework has several limitations. Since our method models multimodal exploration in a coarse anchor space rather than directly reasoning over full future map completion, the utility of distant but weakly indicated regions may not always be captured well. In addition, the downstream planning and clustering stages depend on the quality of the predicted obstacle map and candidate-path structure, so errors in these intermediate representations can affect mode construction and final selection. Future work may benefit from stronger global uncertainty estimation, more topology-aware proposal generation, and tighter coupling between multimodal hypothesis generation and downstream planning.

\paragraph{Acknowledgment.} This work was supported by the National Natural Science Foundation of China (Grant No. 62636008, 62376186, 62472333, and 61932009).






{
\small
\bibliographystyle{unsrtnat}
\bibliography{Styles/example_paper}
}

\clearpage

\section*{NeurIPS Paper Checklist}

\begin{enumerate}

\item {\bf Claims}
    \item[] Question: Do the main claims made in the abstract and introduction accurately reflect the paper's contributions and scope?
    \item[] Answer: \answerYes{}{} 
    \item[] Justification: We clarify the research contributions and scope of the study in the abstract and introduction of the paper.

\item {\bf Limitations}
    \item[] Question: Does the paper discuss the limitations of the work performed by the authors?
    \item[] Answer: \answerYes{} 
    \item[] Justification: Limitations and future directions are discussed in the Conclusion.

\item {\bf Theory assumptions and proofs}
    \item[] Question: For each theoretical result, does the paper provide the full set of assumptions and a complete (and correct) proof?
    \item[] Answer: \answerNo{} 
    \item[] Justification: The main text includes theoretical analysis in Sec.~\ref{sec:cfm_multimodal_anchors}, where Proposition~\ref{prop:deterministic_anchor_collapse} and Proposition~\ref{prop:cfm_anchor_transport} justify the formulation of multimodal anchor modeling and the use of CFM. The main text provides proof sketches; complete proofs are omitted from this version.

    \item {\bf Experimental result reproducibility}
    \item[] Question: Does the paper fully disclose all the information needed to reproduce the main experimental results of the paper to the extent that it affects the main claims and/or conclusions of the paper (regardless of whether the code and data are provided or not)?
    \item[] Answer: \answerNo{} 
    \item[] Justification: The extended network and hyperparameter details are not included in this version.

\item {\bf Open access to data and code}
    \item[] Question: Does the paper provide open access to the data and code, with sufficient instructions to faithfully reproduce the main experimental results, as described in supplemental material?
    \item[] Answer: \answerYes{} 
    \item[] Justification:  The code will be released upon acceptance. The datasets used derive from AimDooM and MP3d; they are already public and available online.

\item {\bf Experimental setting/details}
    \item[] Question: Does the paper specify all the training and test details (e.g., data splits, hyperparameters, how they were chosen, type of optimizer, etc.) necessary to understand the results?
    \item[] Answer: \answerYes{} 
    \item[] Justification: The details necessary to faithfully reproduce our experiments (training procedure, data splits, optimizer, etc.) are included in the paper. 

\item {\bf Experiment statistical significance}
    \item[] Question: Does the paper report error bars suitably and correctly defined or other appropriate information about the statistical significance of the experiments?
    \item[] Answer: \answerYes{} 
    \item[] Justification: We report average results of multiple runs in our experimental section. Our paper does not report error bars.

\item {\bf Experiments compute resources}
    \item[] Question: For each experiment, does the paper provide sufficient information on the computer resources (type of compute workers, memory, time of execution) needed to reproduce the experiments?
    \item[] Answer: \answerNo{} 
    \item[] Justification: This version does not provide detailed hardware and runtime information for each experiment.
    
\item {\bf Code of ethics}
    \item[] Question: Does the research conducted in the paper conform, in every respect, with the NeurIPS Code of Ethics \url{https://neurips.cc/public/EthicsGuidelines}?
    \item[] Answer: \answerYes{} 
    \item[] Justification: We reviewed and ensured that the present work respects the NeurIPS Code of Ethics at each individual part.

\item {\bf Broader impacts}
    \item[] Question: Does the paper discuss both potential positive societal impacts and negative societal impacts of the work performed?
    \item[] Answer: \answerNA{} 
    \item[] Justification:  Our paper is not highly related to societal impacts.
    
\item {\bf Safeguards}
    \item[] Question: Does the paper describe safeguards that have been put in place for responsible release of data or models that have a high risk for misuse (e.g., pretrained language models, image generators, or scraped datasets)?
    \item[] Answer: \answerNA{} 
    \item[] Justification:  This paper does not pose any such risk.

\item {\bf Licenses for existing assets}
    \item[] Question: Are the creators or original owners of assets (e.g., code, data, models), used in the paper, properly credited and are the license and terms of use explicitly mentioned and properly respected?
    \item[] Answer: \answerYes{} 
    \item[] Justification: All the datasets used to train our models, and the code for the papers utilized as benchmarks to evaluate our algorithms are cited.

\item {\bf New assets}
    \item[] Question: Are new assets introduced in the paper well documented and is the documentation provided alongside the assets?
    \item[] Answer: \answerNA{} 
    \item[] Justification: We do not release any new assets with this submission. However, we will make the code as well as trained models publicly available if/when the paper is accepted.

\item {\bf Crowdsourcing and research with human subjects}
    \item[] Question: For crowdsourcing experiments and research with human subjects, does the paper include the full text of instructions given to participants and screenshots, if applicable, as well as details about compensation (if any)? 
    \item[] Answer: \answerNA{} 
    \item[] Justification: This research did not involve any crowdsourcing experiments or studies with human subjects.

\item {\bf Institutional review board (IRB) approvals or equivalent for research with human subjects}
    \item[] Question: Does the paper describe potential risks incurred by study participants, whether such risks were disclosed to the subjects, and whether Institutional Review Board (IRB) approvals (or an equivalent approval/review based on the requirements of your country or institution) were obtained?
    \item[] Answer: \answerNA{} 
    \item[] Justification:  This research did not involve studies with human subjects

\item {\bf Declaration of LLM usage}
    \item[] Question: Does the paper describe the usage of LLMs if it is an important, original, or non-standard component of the core methods in this research? Note that if the LLM is used only for writing, editing, or formatting purposes and does not impact the core methodology, scientific rigorousness, or originality of the research, declaration is not required.
    \item[] Answer: \answerNA{} 
    \item[] Justification: LLM is not an important, original, or non-standard component of the core methods in this research.

\end{enumerate}

\end{document}